\documentclass[conference]{ieee_templates/IEEEtran}
\usepackage{times}
\usepackage{amsmath,amsfonts,amssymb,amsthm}
\usepackage{wrapfig}
\usepackage[linesnumbered,ruled,vlined]{algorithm2e}
\newtheorem{theorem}{Theorem}[section]

\usepackage{graphicx}
\usepackage{float}
\usepackage{microtype}
\usepackage{color}
\usepackage[numbers]{natbib}
\usepackage{multicol}
\usepackage[bookmarks=true]{hyperref}
\usepackage[caption=false]{subfig}
\usepackage[percent]{overpic}
\usepackage[
  subtle
]{savetrees}

\renewcommand{\abovecaptionskip}{0pt}
\usepackage{setspace}
\IEEEoverridecommandlockouts 

\title{
Deep Learning Tubes for Tube MPC
}

\author{David D. Fan$^{1,2}$, Ali-akbar Agha-mohammadi$^{2}$,
and Evangelos A. Theodorou$^{1}$
\thanks{$^{1}$Institute for Robotics and Intelligent Machines, Georgia Institute of Technology, Atlanta, GA, USA}%
\thanks{$^{2}$NASA Jet Propulsion Laboratory, California Institute of Technology, Pasadena, CA, USA}%
}

\begin{document}
\maketitle

%===============================================================================

\begin{abstract}
Learning-based control aims to construct models of a system to use for planning or trajectory optimization, e.g. in model-based reinforcement learning.  In order to obtain guarantees of safety in this context, uncertainty must be accurately quantified.  This uncertainty may come from errors in learning (due to a lack of data, for example), or may be inherent to the system.  Propagating uncertainty forward in learned dynamics models is a difficult problem.  In this work we use deep learning to obtain expressive and flexible models of how distributions of trajectories behave, which we then use for nonlinear Model Predictive Control (MPC).  We introduce a deep quantile regression framework for control that enforces probabilistic quantile bounds and quantifies epistemic uncertainty.  Using our method we explore three different approaches for learning tubes that contain the possible trajectories of the system, and demonstrate how to use each of them in a Tube MPC scheme.  We prove these schemes are recursively feasible and satisfy constraints with a desired margin of probability.  We present experiments in simulation on a nonlinear quadrotor system, demonstrating the practical efficacy of these ideas.
\end{abstract}

\IEEEpeerreviewmaketitle
%===============================================================================

\section{Introduction}
In controls and planning, the idea of adapting to unknown systems and environments is appealing; however, guaranteeing safety and feasibility in the midst of this adaptation is of paramount concern.  The goal of robust MPC is to take into account uncertainty while planning, whether it be from modeling errors, unmodeled disturbances, or randomness within the system itself \cite{bemporad1999robust}.  In addition to safety, other considerations such as optimality, real-time tractability, scalability to high dimensional systems, and hard state and control constraints make the problem more difficult.  In spite of these difficulties, learning-based robust MPC continues to receive much attention \cite{Hewing2018,Wabersich2018a,Ravanbakhsh2018,Hewing2017,Gao2014,fan2016differential,Ostafew2016a,Ostafew2014,Aswani2013,bujarbaruah2019adaptive}.  However, in an effort to satisfy the many competing design requirements in this space, certain restrictive assumptions are often made, which include predetermined error bounds, restricted classes of dynamics models, or fixed parameterizations of the uncertainty.  

Consider the following nonlinear dynamics equation that describes a real system:
\begin{equation}
    x_{t+1} = f(x_t,u_t) + w_t
    \label{eq:x}
\end{equation}
where $x\in\mathbb{X}\subseteq\mathbb{R}^n$ is the state, $u\in\mathbb{U}\subseteq\mathbb{R}^m$ are controls, and $w\in\mathbb{R}^n$ is noise or disturbance.

\begin{figure}[t]
\centering
\includegraphics[trim=2cm 0.5cm 1.5cm 1cm, clip,width=0.7\linewidth]{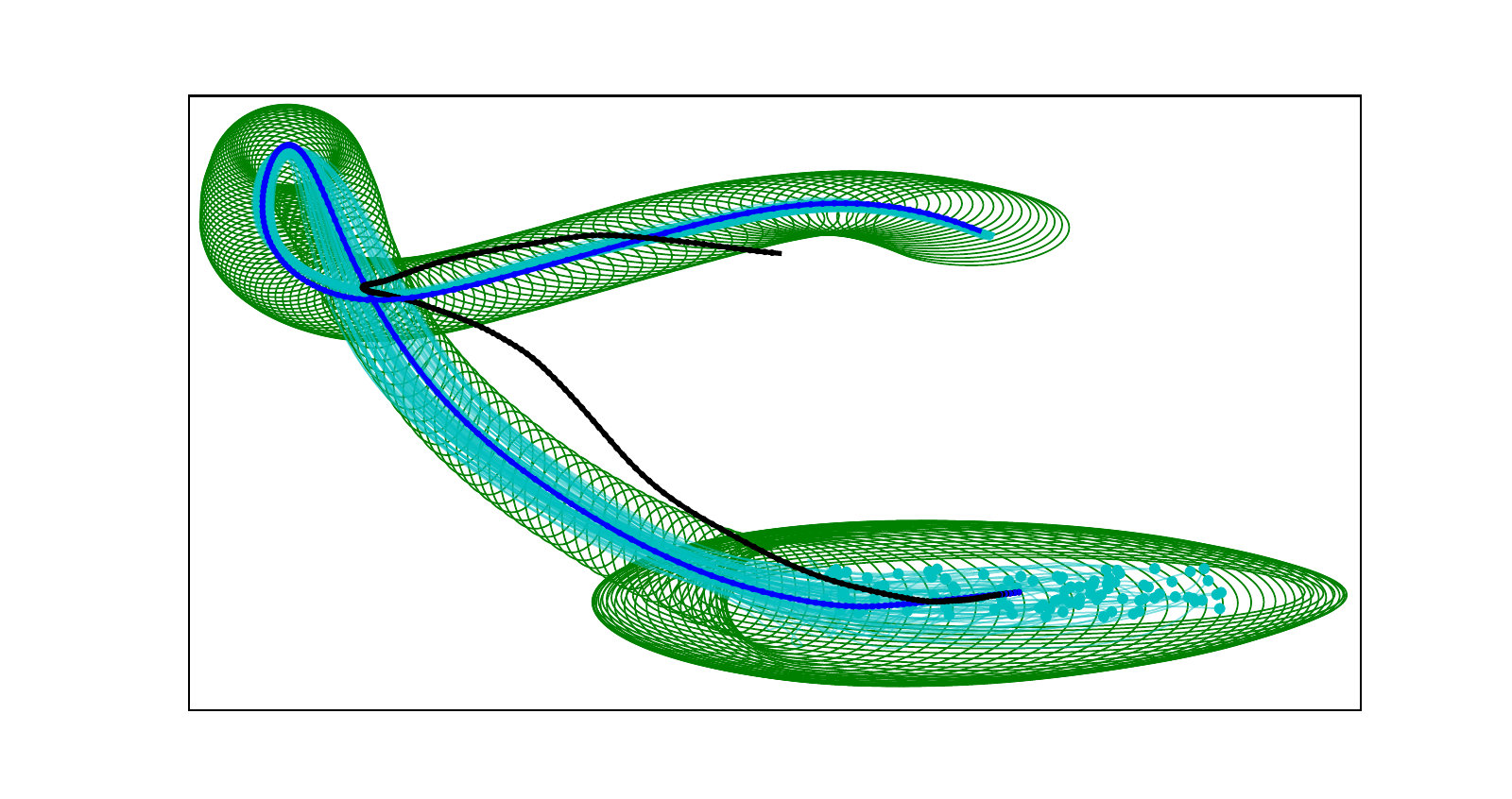}
\caption{A learned tube (green) with learned mean (blue) that captures the distribution of trajectories (cyan) on a full quadrotor model tracking a target trajectory (black), propagated for 200 timesteps forward from the initial states (dots).}
\label{fig:splash}
% \vspace{-0.8cm}
\end{figure}

When attempting to find a model which captures the behavior of $x_t$, there will be error that results from insufficient data, lack of knowledge of $w_t$, or unknown or unobserved higher-dimensional dynamics not observed in $x_t$.  One traditional approach has been to find robust bounds on the model error and plan using this robust model, i.e. $|w_t|\leq W$.  However, this approach can be too conservative since it is not time or space varying and does not capture the distribution of the disturbance \cite{Aswani2013,segu2019general}.  To partially address this one could extend $W$ to be time and state-varying, i.e. $W = W(x_t,u_t,t)$, as is commonly done in the robust MPC and control literature.  For example, \cite{lopez2019dynamic} takes this approach for feedback linearizable systems using boundary layer control, \cite{singh2017robust} leverages contraction theory and sum-of-squares optimization to find stabilizing controllers for nonlinear systems under uncertainty, and \cite{villanueva2017robust} solves for forward invariant tubes using min-max differential inequalities (See \cite{kohler2019robust} for a recent overview of other related approaches).  In this work we aim to learn this uncertainty directly from data, which allows us to avoid structural assumptions of the system of interest or restrictive parameterizations of uncertainty.  We learn a \textit{quantile} representation of the bounds of the distribution of possible trajectories, in the form of a tube around some nominal trajectory (Figure 1).
%However, this requires full a priori knowledge of the system along with its uncertainties.  This assumption is prohibitive when gathering data from a real safety-critical system. In order to take into account both \textbf{risk} (i.e. \textit{aleatoric} uncertainty from variability inherent in the system) and \textbf{uncertainty} (i.e. \textit{epistemic} uncertainty from a lack of sufficient data) \cite{knight2012risk,Roy2011}, a full Bayesian probabilistic modeling approach is necessary.

More closely related to our approach is the wide range of recent work in learning-based planning and control that seeks to handle model uncertainty probabilistically, where a model is constructed from one-step prediction measurements, and it is assumed that the true underlying distribution of the function is Gaussian \cite{fan2019bayesian,hewing2019learning,deisenroth2011pilco,liu2019robust,Berkenkamp2017}:
\begin{equation}
    P(x_{t+1}|x_t,u_t) = \mathcal{N}(\mu(x_t,u_t),\sigma(x_t,u_t)).
\end{equation}
where the mean function $\mu:\mathbb{X}\times\mathbb{U}\rightarrow\mathbb{X}$ and variance function $\sigma:\mathbb{X}\times\mathbb{U}\rightarrow\mathbb{X}^2$ capture the uncertainty of the dynamics for one time step.  Various approaches for approximating this posterior distribution have been developed \cite{girard2003gaussian, gal2016improving}.  For example, in PILCO and related work \cite{hewing2019learning}, moment matching of the posterior distribution is performed to find an analytic expression for the evolution of the mean and the covariance in time. 
% The main problem with this approach is the one-step nature of this model, whereas for trajectory optimization we wish to propagate the distribution over multiple timesteps.
% Consider the two-timestep propagation:
% \begin{multline}
%     P(x_{t+2}|x_t,u_t,u_{t+1})\\
%     % = \int P(x_{t+2}|x_{t+1},u_{t+1})P(x_{t+1}|x_t,u_t)dx_{t+1} \\
%     = \int \big[\mathcal{N}(x_{t+2}|\mu(x_{t+1},u_{t+1}),\sigma(x_{t+1},u_{t+1}))\\
%     \mathcal{N}(x_{t+1}|\mu(x_t,u_t),\sigma(x_t,u_t))\big]dx_{t+1}
% \end{multline}
% The two-timestep distribution will no longer be Gaussian, and marginalizing over the intermediate state $x_{t+1}$ is intractable.  
% As we do so, the true distribution may become multi-modal and highly non-Gaussian, and as the number of timesteps grows, the situation will grow worse.  
However, in order to arrive at these analytic expressions, assumptions must be made which lower the descriptive power for the model to capture the true underlying distribution, which may be multi-modal and highly non-Gaussian.  Furthermore, conservative estimates of the variance of the distribution will grow in an unbounded manner as the number of timesteps increases \cite{Koller2018}.  The result is that any chance constraints derived from these approximate models may be inaccurate.  In Figure \ref{fig:distributions} we compare the classic GP-based moment matching approach for propagating uncertainty with our own deep quantile regression method on two different functions.  While GP-moment matching can both underestimate and overestimate the true distribution of trajectories, our method is less prone to failures due to analytic simplifications or assumptions.

\begin{figure}[tb]
\centering
\subfloat{
    \includegraphics[width=0.5\linewidth,trim={40 5 20 20},clip]{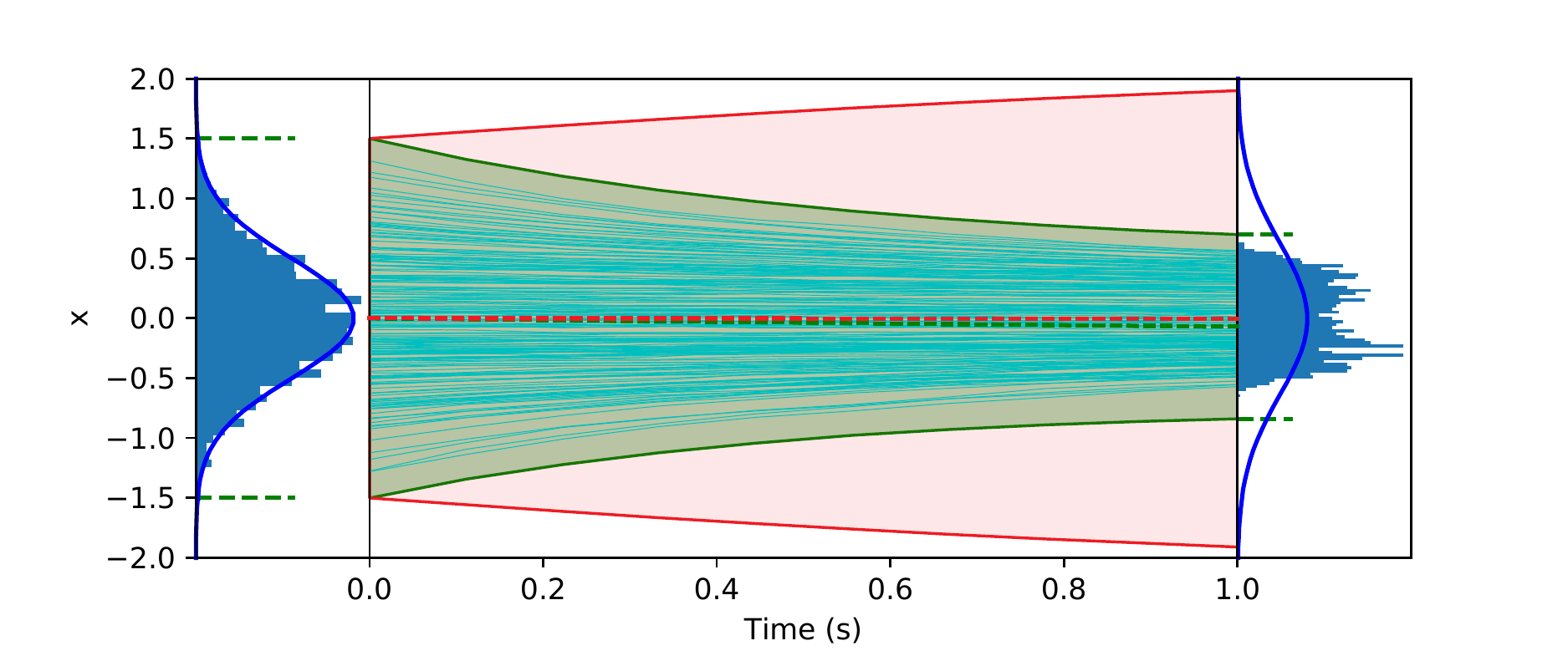}
}
\subfloat{
    \includegraphics[width=0.5\linewidth,trim={40 5 20 20},clip]{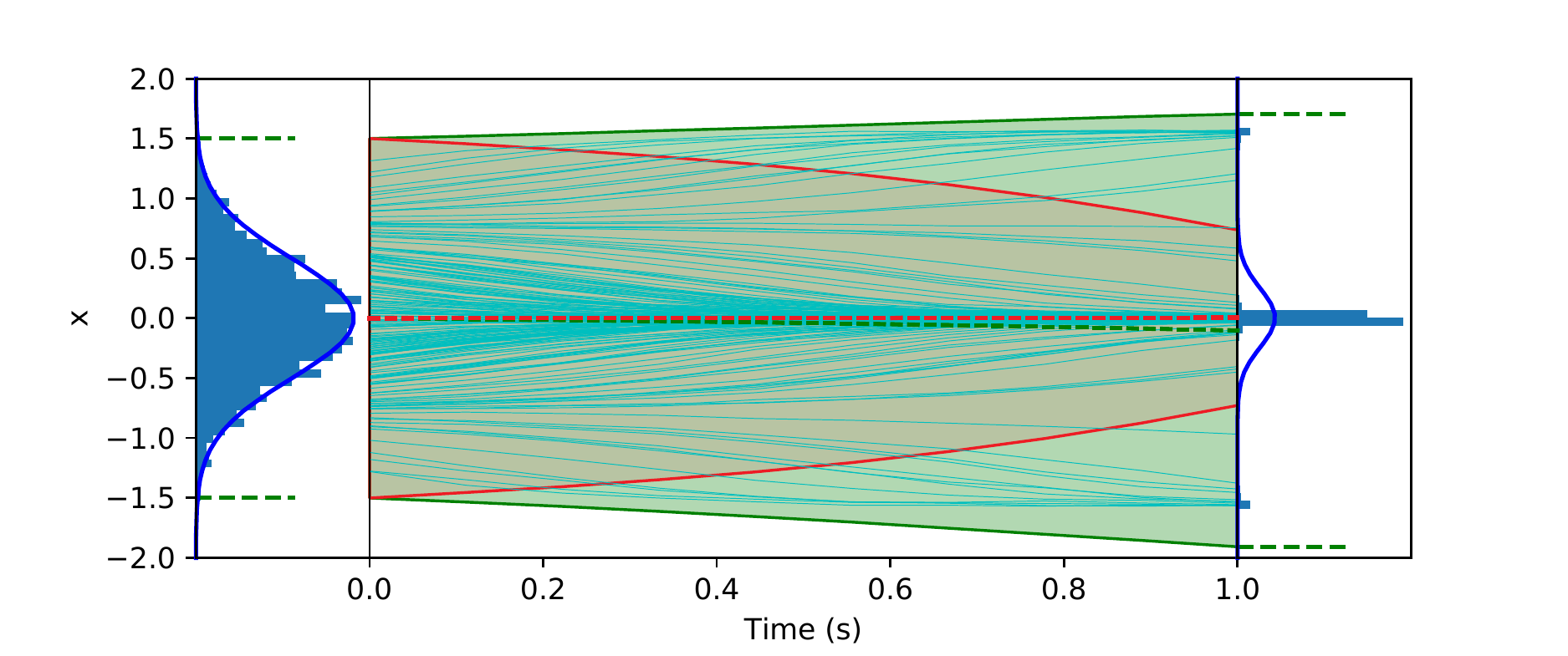}
}
\caption{Comparison of 3-$\sigma$ bounds on distributions of trajectories using GP moment matching (red) and the proposed quantile regression method (green).  100 sampled trajectories are shown (cyan) along with starting and ending distributions (blue, left and right histograms).  Left:  GP moment matching overestimates the distribution for the dynamics $\dot{x} = -x|x|$, while our method models it well.  Right:  GP moment matching underestimates the distribution for the dynamics $\dot{x}=-\sin(4x)$, while our method captures the tails of the distribution.}
\label{fig:distributions}
\end{figure}

An alternative approach to Bayesian modeling for robust MPC has been to use quantile bounds to bound the tails of the distribution.  This has the advantage that for planning in safety-critical contexts, we are generally not concerned with the full distribution of the trajectories, but the tails of these distributions only; specifically, we are interested in the probability of the tail of the distribution violating a safe set.  A few recent works have taken this approach in the context of MPC; for example, \cite{bradford2019nonlinear} computes back-off sets with Gaussian Processes, and \cite{bujarbaruah2019adaptive} uses an adaptive control approach to parameterize quantile bounds.

We are specifically interested in the idea of learning quantile bounds using the expressive power of deep neural networks.  Quantile bounds give an explicit probability of violation at each timestep and allow for quantifying uncertainty which can be non-Gaussian, skewed, asymmetric, multimodal, and heteroskedastic \cite{tagasovska2019single}.  Quantile regression itself is a well-studied field with the first results from \cite{koenker1978regression}, see also \cite{koenker2001quantile,taylor1999quantile}.  Quantile regression in deep learning has been also recently considered as a general statistical modeling tool \cite{rodrigues2018beyond,zhang2019regression,sadeghi2019efficient,yan2018parsimonious, tagasovska2019single}.  Bayesian quantile regression has also been studied \cite{kozumi2011gibbs,yang2016posterior}.  Recently quantile regression has gained popularity as a modeling tool within the reinforcement learning community \cite{dabney2018distributional}.

In addition to introducing a method for deep learning quantile bounds for distributions of trajectories, we also show how this method can be tailored to a tube MPC framework.  Tube MPC \cite{Langson2004,mayne2014model} was introduced as a way to address some of the shortcomings of classic robust MPC; specifically that robust MPC relied on optimizing over an open-loop control sequence, which does not predict the closed-loop behavior well.  Instead, tube MPC seeks to optimize over a local policy that generates some closed-loop behavior, which has advantages of robust constraint satisfaction, computational efficiency, and better performance.  The use of tube MPC allows us to handle high dimensional systems, as well as making the learning problem more efficient, tractable, and reliable.  To the best of our knowledge, our work is the first to combine deep quantile regression with tube-based MPC, or indeed any learning-based robust MPC method.

The structure of the paper is as follows:  In Section \ref{sec:2} we present our approach for learning tubes, which includes deep quantile regression, enforcing a monotonicity condition with a negative divergence loss function, and quantifying epistemic uncertainty.  In Section \ref{sec:3} we present three different learning tube MPC schemes that take advantage of our method.  In Section \ref{sec:4} we perform several experiments and studies to validate our method, and conclude in Section \ref{sec:5}.
%===============================================================================1

\section{Deep Learning Tubes}
\label{sec:2}
\subsection{Learning Tubes For Robust and Tube MPC}
\begin{figure}[tb]
\centering
% \subfloat{
\includegraphics[trim=0cm 9.3cm 16.9cm 0cm,clip,width=0.55\linewidth]{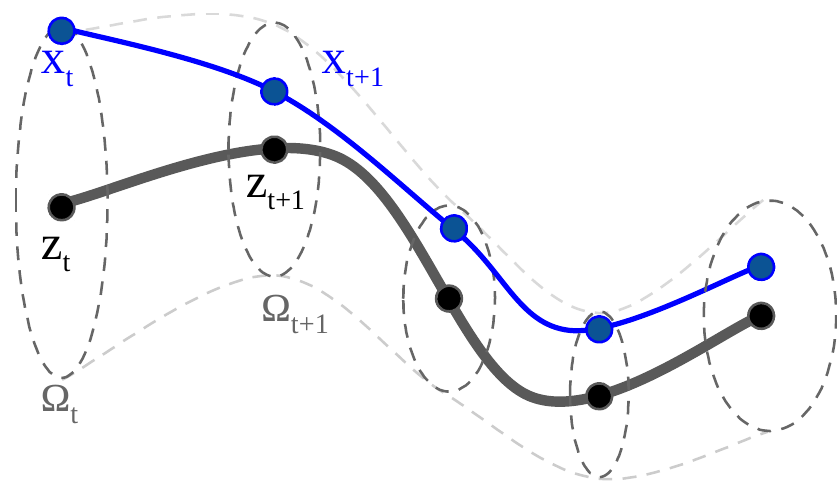}
% }
% \subfloat{
% \includegraphics[trim=1.8cm 1cm 0.5cm 1.1cm,clip,width=0.49\linewidth]{figures/learned_tube_closeup.pdf}
% }
\caption{Diagram of a tube around the dynamics of $z$, within which $x$ stays invariant.  Note that the tube set $\Omega_t$ is time-varying.}
% Right:  Learned tube dynamics on a triple-integrator system.  Blue line is $x$, red line is $z$, red circles indicate actual tracking error, green circles indicate predicted tube.}
\label{fig:tube}
\end{figure}
We propose learning time-varying invariant sets as a way to address the difficulties with propagating uncertainty for safety critical control, as well as to characterize the performance of a learned model or tracking controller.  Consider the following \textit{quantile} description of the dynamics:
\begin{align}
    \label{eq:tube_dyn}
    x_{t+1} &= f(x_t,u_t) + w_t\\\nonumber
    z_{t+1} &= f_z(z_t,v_t)\\\nonumber
    \omega_{t+1} &= f_\omega(\omega_t,z_t,v_t,t)\\\nonumber
    P(d(x_t,z_t)& \leq \omega_t) \geq \alpha, \quad\forall t\in\mathbb{N}
\end{align}
where $z\in\mathbb{Z}\subseteq\mathbb{R}^{n_z}$ is a latent state of equal or lower dimension than $x$, i.e. $n_z\leq n$, and $v\in\mathbb{V}\subseteq\mathbb{R}^{m_z}$ is a pseudo-control input, also of equal or lower dimension than $u$, i.e. $m_z\leq m$.  In the simplest case, we can fix $v_t = u_t$ and/or $z_t = x_t$.  Also, $\omega\in\mathbb{R}^{n_z}$ is a vector that we call the \textit{tube width}, with each element of $\omega>0$.

This defines a "tube" around the trajectory of $z$ within which $x$ will stay close to $z$ with probability greater than $\alpha\in[0,1]$ (Figure \ref{fig:tube}).   More formally, we can define the notion of closeness between some $x$ and $z$ by, for example, the distance between $z$ and the projection of $x$ onto $\mathbb{Z}$:  $d(x,z)=|P_\mathbb{Z}(x)-z|\in\mathbb{R}^{n_z}$, where $P_\mathbb{Z}$ is a projection operator.  Let $\Omega_\omega(z)\subset\mathbb{X}$ be a set in $\mathbb{X}$ associated with the tube width $\omega$ and $z$:
\begin{equation}
    \Omega_{\omega}(z):=\{x\in\mathbb{X}: d(x,z) \leq \omega\}.
\end{equation}
where the $\leq$ is element-wise.  Other tube parameterizations are possible, for example $\Omega_{\omega}(z):=\{x\in\mathbb{X}:\|P_\mathbb{Z}(x)-z\|_\omega\leq 1\}$, where $\omega\in\mathbb{R}^{n_z\times n_z}$ instead.

The coupled system (\ref{eq:tube_dyn}) induces a sequence of sets $\{\Omega_{\omega_t}(z_t)\}_{t=0}^T$ that form a tube around $z_t$.  Our goal is to learn how this tube changes over time in order to use it for planning safe trajectories.  

\subsection{Quantile Regression}
Our challenge is to learn the dynamics of the tube width, $f_\omega$.  Given data collected as trajectories $\mathcal{D}=\{x_t,u_t,x_{t+1},z_t,v_t,z_{t+1},t\}_{t=0}^T$, we can formulate the learning problem for $f_\omega$ as follows.

Let $f_\omega$ be parameterized with a neural network,  $f_\omega^\theta$.  \sloppy For a given $t$ and data point $\{x_t,u_t,x_{t+1},z_t,v_t,z_{t+1},t\}$, let $\omega_t=d(x_t,z_t)$ be the input tube width to $f_\omega$, and $\omega_{t+1}=d(x_{t+1},z_{t+1})$ the candidate output tube width.  The candidate tube width at $t+1$ must be less than the estimate of the tube width at $t+1$, i.e: $\omega_{t+1}\leq f_\omega^\theta(\omega_t,z_t,v_t,t)$.  To train the network $f_\omega^\theta$ to respect these bounds we can use the following \textit{check loss} function:
\begin{align}
    L_\omega^\alpha(\theta,\delta) &=  L^\alpha(\omega_{t+1},f_\omega^\theta(\omega_{t},z_{t},v_{t},{t}))\\   L^\alpha(y,r)&=\begin{cases}
               \alpha|y-r| \quad\quad\quad y> r\\
               (1-\alpha)|y-r| \quad y\leq r
            \end{cases}\nonumber
    \label{eq:tube_loss}
\end{align}
where the loss is a function of each data sample $\delta=\{\omega_{t+1},\omega_t,z_t,v_t,t\}$.   With the assumption of i.i.d. sampled data, when $L_\omega^\alpha(\theta,\delta)$ is minimized the quantile bound will be satisfied, (see Figure \ref{fig:tube_learn} and Theorem \ref{thm:bound_loss}).
In practice we can smooth this loss function near the inflection point $y=r$ with a slight modification, by multiplying $L_\omega^\alpha$ with a Huber loss \cite{huber1992robust,dabney2018distributional}.

\begin{figure}[tb]
\centering
\includegraphics[trim=0cm 9.5cm 8.5cm 0cm, clip,width=\linewidth]{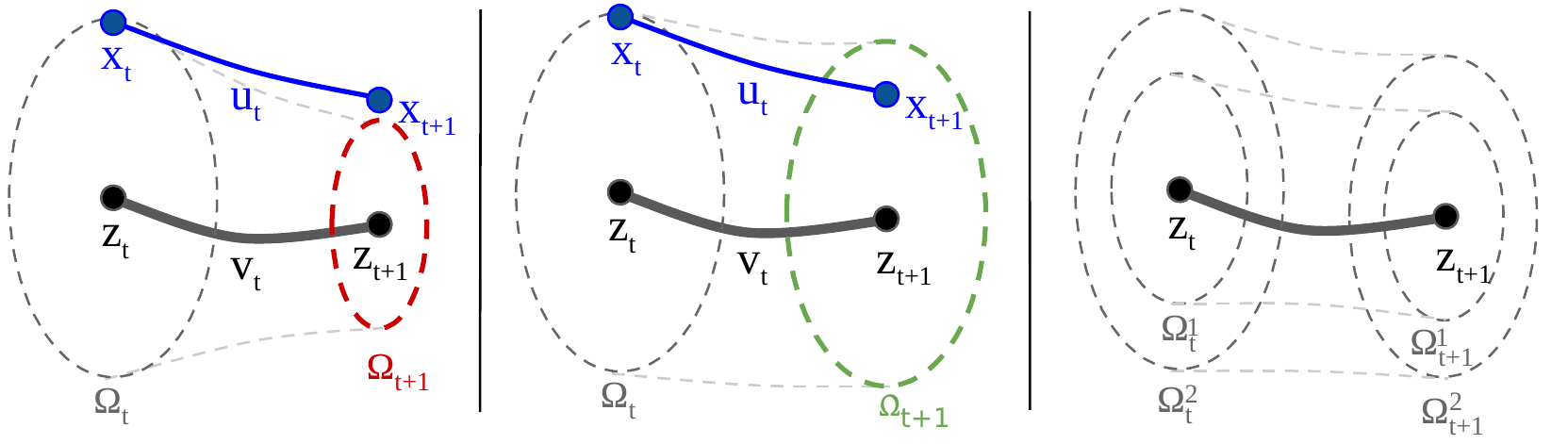}
\caption{Learning tube dynamics from data.  Left:  The predicted tube at $t+1$ is too small.  The gradient of the loss function will increase its size.  Middle:  The predicted tube at $t+1$ is larger than the actual trajectory in $x$ taken, and will be shrunk.  Right:  The mapping $f_\omega(\omega,z_t,v_t,t)$ is monotonic with respect to $\omega$, which results in $\Omega_t^1\subseteq\Omega_t^2 \implies \Omega_{t+1}^1\subseteq\Omega_{t+1}^2$.}
\label{fig:tube_learn}
\end{figure}

\begin{theorem}
Let $\theta^*$ minimize $\mathbb{E}_\delta [L_\omega^\alpha(\theta,\delta)]$.  Then with probability $\alpha$, $f_\omega^{\theta^*}(\omega,z,v,t)$ is an upper bound for $f_\omega(\omega,z,v,t)$.
\label{thm:bound_loss}
\end{theorem}
\begin{proof}
With a slight abuse of notation, let $x$ denote the input variable to the loss function, and consider the expected loss $\mathbb{E}_x[L^\alpha(y(x),r(x))]$.  We find the minimum of this loss w.r.t. $r$ by setting the gradient to 0:
\begin{align}
    & \frac{\partial}{\partial r^*} \mathbb{E}_x[L^\alpha(y(x),r^*(x))]\\\nonumber
    \quad &=\int_{y(x)> r^*(x)}\alpha p(x)dx - \int_{y(x)\leq r^*(x)}(1-\alpha)p(x)dx\\\nonumber
    \quad &=\alpha p(y(x)> r^*(x)) - (1-\alpha)p(y(x)\leq r^*(x)) = 0\\\nonumber
    \quad &\implies p(y(x)\leq r^*(x)) = \alpha
\end{align}
Replacing $r^*(x)$ with $f_\omega^{\theta^*}(\omega,z,v,t)$ and $y(x)$ with $f_\omega(\omega,z,v,t)$ completes the proof.
\end{proof}

Note that quantile regression gives us tools for learning tube dynamics $f_\omega(\omega,z,v,t,\alpha)$ that are a function of the quantile probability $\alpha$ as well.  This opens the possibility to dynamically varying the margin of safety while planning, taking into account acceptable risks or value at risk \cite{gaglianone2011evaluating}.  For example, in planning a trajectory, one could choose a higher $\alpha$ for the near-term and lower $\alpha$ in the later parts of the trajectory, reducing the conservativeness of the solution.

Additionally, we note that we can train the tube bounds dynamics in a recurrent fashion to improve long sequence prediction accuracy.  While we present the above and following theorems in the context of one timestep, they are easily extensible to the recurrent case.

\subsection{Enforcing Monotonicity}
In addition to the quantile loss we also introduce an approach to enforce monotonicity of the tube with respect to the tube width (Figure \ref{fig:tube_learn}, right).  This is important for ensuring recursive feasibility of the MPC problem, as well as allowing us to shrink the tube width during MPC at each timestep if we obtain measurement updates of the current state, or, in the context of state estimation, an update to the covariance of the estimate of the current state.  Enforcing monotonicity in neural networks has been studied with a variety of techniques \cite{sill1998monotonic,you2017deep}.  Here we adopt the approach of using a loss function that penalizes the network for having negative divergence, similar to \cite{gupta2019monotonic}:
\begin{equation}
    L_m(\theta,\delta) = -\min(0,\textrm{div}_\omega f_\omega(\omega,z,v,t))
\end{equation}
where $\textrm{div}_\omega$ is the divergence of $f_\omega$ with respect to $\omega$.
In practice we find that under gradient-based optimization, this loss decreases to 0 in the first epoch and does not noticeably affect the minimization of the quantile loss.  Minimizing $L_m(\theta, \delta)$ allows us to make claims about the monotonicity of the learned tube:

\begin{theorem}
Suppose $\theta^*$ minimizes $\mathbb{E}_\delta[L_m(\theta,\delta)]$ and $\mathbb{E}_\delta[L_m(\theta^*,\delta)]=0$.  Then for any $z_t\in\mathbb{Z},\, v_t\in\mathbb{V},\,t\in\mathbb{N}$ and $\omega^1_t,\omega^2_t\in\mathbb{R}^{n_z}$, if $\Omega_{\omega^1_t}(z_t)\subseteq\Omega_{\omega^2_t}(z_t)$, then $\Omega_{\omega^1_{t+1}}(z_{t+1})\subseteq\Omega_{\omega^2_{t+1}}(z_{t+1})$.
\label{thm:monotone}
\end{theorem}
\begin{proof}
Since $\forall \theta,\delta,\, L_m(\theta,\delta)>0$ and $\mathbb{E}[L_m(\theta^*,\delta)]=0$, then $L_m(\theta^*,\delta)=0$.  Then $\nabla_\omega f_\omega(\omega,z,v,t) > 0$ and $f_\omega$ is nondecreasing with respect to $\omega$.  Since $\Omega_{\omega^1_t}\subseteq\Omega_{\omega^2_t}$, then $\omega^1_t\leq\omega^2_t$, so $f_\omega(\omega^1_t,z_t,v_t,t)\leq f_\omega(\omega^2_t,z_t,v_t,t)$, which implies that $\Omega_{\omega^1_{t+1}}(z_{t+1})\subseteq\Omega_{\omega^2_{t+1}}(z_{t+1})$.
\end{proof}

\subsection{Epistemic Uncertainty}
Finally, in order to account for uncertainty in regions where no data is available for estimating quantile bounds, we incorporate methods for estimating epistemic uncertainty.  Such methods can include Bayesian neural networks, Gaussian Processes, or other heuristic methods in deep learning \cite{gal2016dropout,dabney2018distributional,rasmussen2003gaussian}.  For the experiments in this work we adopt an approach that adds an additional output layer to our quantile regression network that is linear with respect to orthonormal weights \cite{tagasovska2019single}.  We emphasize that a wide range of methods for quantifying epistemic uncertainty are available and we are not restricted to this one approach; however, for the sake of clarity, we present in detail our method of choice.  Let $g(z,v,t)$ be a neural network with either fixed weights that are either randomly chosen or pre-trained, with $l$ dimensional output.  We branch off a second output with a linear layer: $C^\intercal g(z,v,t)$, where $C\in\mathbb{R}^{l\times k}$.  The estimate of epistemic uncertainty is chosen as $u_e(z,v,t) = \|C^\intercal g(z,v,t)\|^2$.  Then, the parameters $C$ are trained by minimizing the following loss:
\begin{equation}
    L_u(C,\delta) = \|C^\intercal g(z,v,t)\|^2 + \lambda \|C^\intercal C - I_k\|.
\end{equation}
where $\lambda>0$ weights the orthonormal regularization.  Minimizing this loss produces a network that has a value close to 0 when the input data is in-distribution, and increases with known rate as the input data moves farther from the training distribution (Figure \ref{fig:epistemic_example}, and see \cite{tagasovska2019single} for detailed analysis).  We scale the predicted quantile bound by the epistemic uncertainty, then add a maximum bound to prevent unbounded growth as $\omega$ grows: 
\begin{equation}
f_\omega(\omega,z,v,t) \leftarrow \min\{(1 + \beta u_e(z,v,t))f_\omega(\omega,z,v,t),W\}    
\end{equation}
where $\beta>0$ is a constant parameter that scales the effect of the epistemic uncertainty, and $W$ is a vector that provides an upper bound on the total uncertainty.  Finding an optimal $\beta$ analytically may require some assumptions such as a known Lipschitz constant of the underlying function, non-heteroskedastic noise, etc., which we leave for future investigation.  We set $\beta$ and $W$ by hand and find this approach to be effective in practice.  %Note that we also exclude $\omega$ from the input of $g(z,v,t)$, in order to prevent unbounded growth of the epistemic uncertainty as it is propagated forward in time.  

We expect that as the field matures, methods for providing guarantees on well-calibrated epistemic uncertainty in deep learning will continue to improve.  In the meantime, we make the assumption that we have well-calibrated epistemic uncertainty, an assumption similar to those made with other learning-based controls methods, such as choosing noise covariances, disturbance magnitudes, or kernel types and widths.  The main benefit of leveraging epistemic uncertainty modeling is that it allows us to maintain guarantees of safety and recursive feasibility when we have a limited amount of data to learn from.  In the case when no reliable epistemic estimate is available, we can proceed if we simply assume there is sufficient data to learn a good model offline.

\begin{figure}[tb]
\centering
\includegraphics[trim=0.5cm 0cm 0.5cm 0.5cm, clip,width=0.8\linewidth]{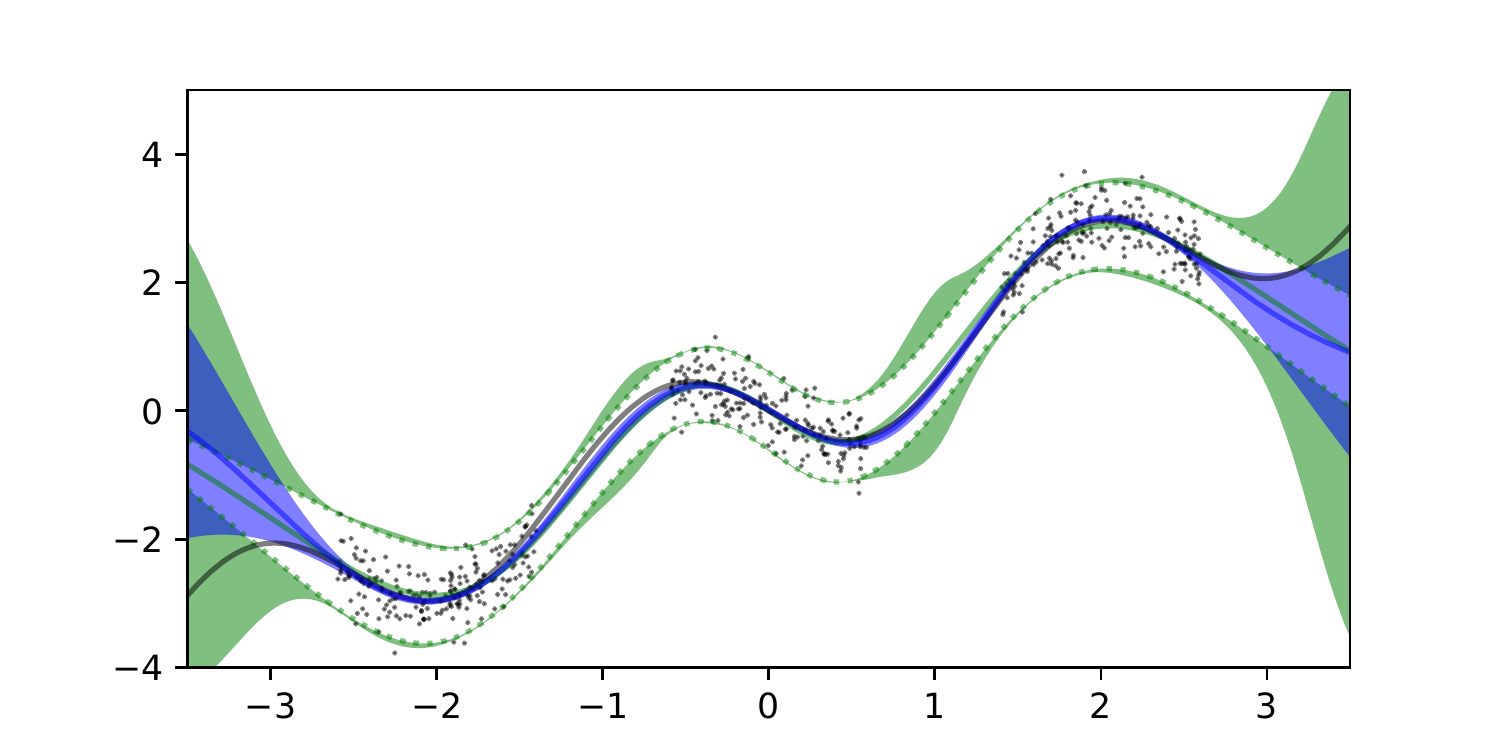}
\caption{Estimating epistemic uncertainty for a 1-D function.  Black dots indicate noisy data used to train the models, black line indicates the true function.  Green colors indicate trained neural network models with green line indicating mean, and green dotted lines indicating learned 99\% quantile bounds.  Green shading indicates increased quantile bounds scaled by the learned epistemic uncertainty.  Blue line and shading is GP regression with 99\% bounds for comparison.}
\label{fig:epistemic_example}
\end{figure}

\section{Three Ways to Learn Tubes for Tube MPC}
\label{sec:3}
In this section we present three variations for applying our deep quantile regression approach to MPC problems, whose applicability may vary based on what components are available to the designer.  By leveraging the previously described theorems for ensuring accurate quantiles, monotonicity, and uncertainty of the tube width dynamics, we can guarantee recursive feasibility of these MPC schemes, while ensuring that the trajectory of the system $x_t$ remains within a safe set $x_t\in\mathcal{C}\subset\mathbb{R}^n$ with probability $\alpha$ at each timestep.  The three different approaches require different elements of the system to be known or given, and are summarized as: 
\begin{enumerate}
    \item Given a tracking control law $u=\pi(x,z)$ and reference trajectory dynamics $f_z$, construct an invariant tube with the reference trajectory at its center (Figure \ref{fig:tube}).
    \item Given a tracking control law $\pi$ and reference trajectory dynamics $f_z$, construct a model of the dynamics of the error $e = x - z$, then learn an invariant tube with $z + e$ as its center (Figure \ref{fig:tracking_tube_figure}).
    \item From data generated from any control law, random or otherwise, learn a reduced representation of the dynamics $f_z$ (and optionally, a policy $\pi$ to track it), along with tube bounds on the tracking error (Figure \ref{fig:learn_model_error}).
\end{enumerate}

\subsection{Learning Tube Dynamics for a Given Controller}
We first consider the case where we are given a fixed ancillary controller $\pi(x,z):\mathbb{X}\times\mathbb{Z}\rightarrow\mathbb{U}$ (or potentially $\pi(x,z,v)$ with a feed-forward term $v$), along with nominal dynamics $f_z$ that are used for planning and tracking in the classic tube MPC manner \cite{Mayne2011}.  For now our goal is to learn $f_\omega$ alone.

We sum the three losses discussed in the previous section:
\begin{equation}
    L(\theta,C,\delta) = L_\omega^\alpha(\theta,\delta) + L_m(\theta,\delta) + L_u(C,\delta)
    \label{eq:total_loss}
\end{equation}
to learn $f^\theta_\omega$, and find $\theta^*$ and $C^*$ via stochastic gradient descent. Next, we perform planning on the coupled $z$ and tube dynamics in the following nonlinear MPC problem.  Let $T\in\mathbb{N}$ denote the planning horizon.  We use the subscript notation $v_{k|t}$ to denote the variable $v_k$ for $k=0,\cdots,T$ within the MPC problem at time $t$.  Let $v_{\cdot|t}$ denote the set of variables $\{v_{k|t}\}_{k=0}^T$.  Then, at time $t$, the MPC problem is:
\begin{subequations}
\label{eq:tube_mpc}
  \begin{align}
    \min_{v_{\cdot|t}\in\mathbb{V}} & J_T(v_{\cdot|t},z_{\cdot|t},\omega_{\cdot|t})\\
    s.t. \,\, \forall k=0,&\cdots,T: \nonumber\\
     z_{k+1|t}&=f_z(z_{k|t},v_{k|t}) \\
     \omega_{k+1|t}&=f_\omega^\theta(\omega_t,z_t,v_t,t)\\
    \omega_{0|t} &= d(x_{t},z_{0|t}) \label{eq:tube_mpc:initial}\\
    z_{T|t} &= f_z(z_{T|t},v_{T|t}) \label{eq:tube_mpc:terminal1}\\
    \omega_{T|t} & \geq f_\omega^\theta(\omega_{T|t},z_{T|t},v_{T|t},T)\label{eq:tube_mpc:terminal2}\\
    \Omega_{\omega_{k|t}}(z_{k|t})&\subseteq\mathcal{C}
  \end{align}
\end{subequations}
Let $v^*_{\cdot|t},z^*_{\cdot|t}$ denote the minimizer of the problem at time $t$.  Note that we include $\omega_{\cdot|t}$ in the cost, which allows us to encourage larger or smaller tube widths.  The tube width $\omega_{0|t}$ is updated based on a measurement $x_t$ from the system, or can also be updated with information from a state estimator.  In the absence of measurements we can also carry over the past optimized tube width, i.e. $\omega_{0|t}=\omega^*_{1|t-1}$, as long as $x_t\in\Omega_{\omega_{0|t}}(z_{0|t})$.  The closed-loop control is set to $v_t=v^*_{0|t}$ and the tracking target for the underlying policy is $z_{t+1}=z^*_{1|t}$.  Under these assumptions we have the following theorem establishing recursive feasibility and safety:
% \begin{subequations}
% \label{eq:tube_mpc}
%   \begin{align}
%     \min_{u_{\cdot|t}\in\mathbb{U}} & J_T(x_{\cdot|t},u_{\cdot|t})\\
%     s.t. \,\, \forall k=0,&\cdots,T: \nonumber\\
%      x_{k+1|t}&=f(x_{k|t},u_{k|t}) \\
%      x_{T|t}&\subseteq \mathbb{X}_f \\
%      x_{k|t}&\subseteq\mathcal{C}\\
%      u_{k|t}&\in\mathbb{U}
%   \end{align}
% \end{subequations}

\begin{theorem}
\label{thm:mpc}
Suppose that the MPC problem (\ref{eq:tube_mpc}) is feasible at $t=0$.  Then the problem is feasible for all $t>0\in\mathbb{N}$ and at each timestep the constraints are satisfied with probability $\alpha$.
\end{theorem}
\begin{proof}
The proof is similar to that in \cite{kohler2019robust} for general set-based robust adaptive MPC.  Let $z_{0|t+1} = z^*_{1|t}$ and choose any $\omega_{0|t+1}$ such that $x_{t+1}\in\Omega_{\omega_{0|t+1}}(z_{0|t+1})$ (if measurements $x_{t+1}$ are unavailable, one can use $\omega_{0|t+1}=\omega^*_{1|t}$).  With probability $\alpha$, $\Omega_{\omega_{0|t+1}}(z_{0|t+1}) \subseteq \Omega_{\omega^*_{1|t}}(z^*_{1|t})$ due to Theorem \ref{thm:bound_loss}.   Let $v_{k|t+1}=v^*_{k+1|t}$ for $k=0,\cdots,T-1$, and let $v_{T|t+1}=v^*_{T|t}$.  Then $v_{\cdot|t+1}$ is a feasible solution for the MPC problem at $t=1$, due to the terminal constraints (\ref{eq:tube_mpc:terminal1},\ref{eq:tube_mpc:terminal2}) as well as the monotonicity of $f_\omega$ with respect to $\omega$ (Theorem \ref{thm:monotone}).  
\end{proof}

Since $f_\omega^\theta(\omega_t,z_t,v_t)$ is nonlinear we find solutions to the MPC problem via iterative linear approximations, yielding an SQP MPC approach \cite{diehl2009efficient,camacho2004inductive}.  Other optimization techniques are possible, including GPU-accelerated sampling-based ones \cite{williams2018information}.  We outline the entire procedure in Algorithm \ref{alg:tube_alg}.  

\begin{algorithm}[tb]
\small
\textbf{Require:} Ancillary policy $\pi$, Latent dynamics $f_z$,  Safe set $\mathcal{C}$, Quantile probability $\alpha$.  MPC horizon $T$.\\
\textbf{Initialize:} Neural network for tube dynamics $f_\omega^\theta$.  Dataset $\mathcal{D}=\{x_{t_i},u_{t_i},x_{t_i+1},z_{t_i},v_{t_i},z_{t_i+1},t_i\}_{i=1}^N$.  Initial states $x_0$, $z_0$, Initial feasible controls $v_{\cdot|0}$. \\
\For{$t=0,\cdots$}{
\If{updateModel}{
    Train $f_\omega^\theta$ on dataset $\mathcal{D}$ by minimizing tube dynamics loss (\ref{eq:total_loss}).
}
\If{$x_{t}$ measured}{
    Initialize tube width $\omega_{0|t}=d(x_t,z_t)$ 
}
Solve MPC problem (\ref{eq:tube_mpc}) with warm-start $v_{\cdot|t}$, obtain $v_t$, $z_{t+1}$\\
Apply control policy to system $u_t=\pi(x_t,z_{t+1})$\\
Step forward for next iteration: $v_{k|t+1}=v^*_{k+1|t},\ k=0,\cdots,T-1,\ v_{T|t+1}=v^*_{T|t},\ z_{0|t+1}=z^*_{1|t},\ \omega_{0|t+1}=\omega^*_{1|t}$\\
Append data to dataset $\mathcal{D}\leftarrow\mathcal{D}\cup\{x_t,u_t,x_{t+1},z_t,v_t,z_{t+1},t\}$
}
\caption{Tube Learning for Tube MPC}
\label{alg:tube_alg}
\end{algorithm}

%===============================================================================

\subsection{Learning Tracking Error Dynamics and Tube Dynamics}
Next we show how to learn error dynamics $e_{t+1}=f_e(e_t,z_t,v_t)$ along with a tube centered along these dynamics, where  $e_t = P_\mathbb{Z}(x) - z$ is the error between $x$ and $z$, with $x$ projected onto $\mathbb{Z}$.  These error dynamics function as the mean of the distribution of dynamics $x_{t+1}=f(x_t,u_t)$ when the tracking policy is used $u_t=\pi(x_t,z_{t+1},v_t)$.  This allows the tube to take on a more accurately parameterized shape (Figure \ref{fig:tracking_tube_figure}).  Setting up the learning problem in this way offers several distinct advantages.  First, rather than relying on an accurate nominal model $f_z$ and learning the bounds between this model and the true dynamics, we directly characterize the difference between the two models with $f_e$.  This means that $f_z$ can be chosen more arbitrarily and does not need to be a high-fidelity dynamics model.  Second, using the nominal dynamics $z_t$ as an input to $f_e$ and learning the error "anchors" our prediction of the behavior of $x_t$ to $z_t$.  This allows us to predict the expected distribution of $x_t$ with much higher accuracy for long time horizons, in contrast to the approach of learning a model $f$ directly and propagating it forward in time, where the error between the learned model and the true dynamics tends to increase with time.

\begin{figure}[tb]
    \centering
    \subfloat[Learning Tracking Error\label{fig:tracking_tube_figure}]{{\includegraphics[trim=0.1cm 11cm 16.5cm 0.1cm, clip,width=0.59\linewidth]{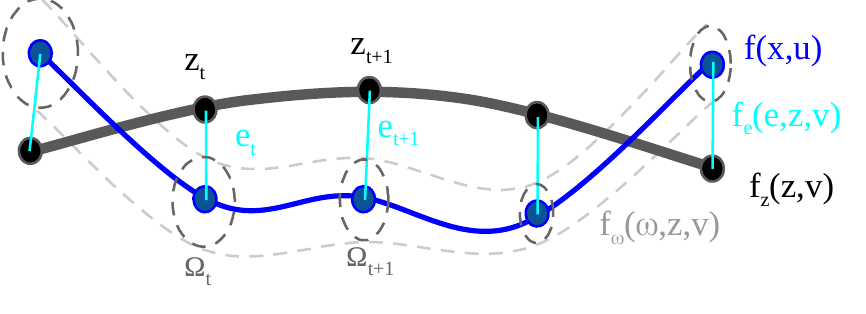}}}
    \vrule\ 
    \subfloat[Learning Model Error\label{fig:learn_model_error}]{{\includegraphics[trim=0cm 10.5cm 19cm 0.1cm, clip,width=0.39\linewidth]{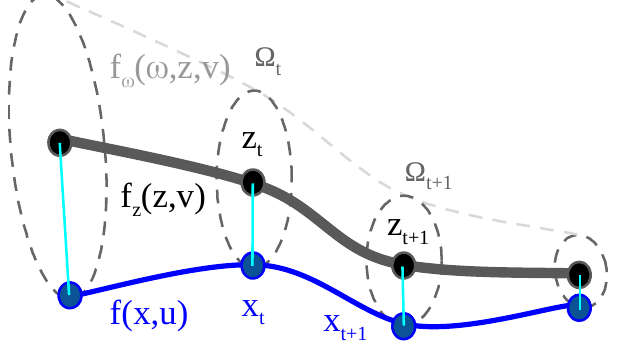}}}
    \caption{{(a) Learning error dynamics $f_e$ along with tube dynamics $f_\omega$.  Black line is the nominal trajectory $f_z$, blue line is data collected from the system.  Cyan indicates tracking errors, whose dynamics are learned.  Grey tube denotes $f_\omega$, which captures the error between the true dynamics and $z_t + e_t$.  (b) Fitting learned dynamics to actual data.  Blue inline indicates data collected from the system, black line is a learned dynamics trajectory fitted to the data.}}
\end{figure}

As before, we assume we have a known $\pi$ and nominal dynamics $f_z$.  Let $\Omega^e_\omega(z,e)\subset\mathbb{X}$ be a set in $\mathbb{X}$ associated with the tube width $\omega$,$z$, and $e$:
\begin{equation}
    \Omega^e_{\omega}(z,e):=\{x\in\mathbb{X}: d(x,z+e) \leq \omega\}.
\end{equation}
where the $\leq$ is element-wise.  We have the following description of the error dynamics:
\begin{align}
% z_{t+1} &= f_z(z_t,v_t)\\\nonumber
e_{t+1} &= f_e(e_t,z_t,v_t)\\\nonumber
\omega_{t+1} &= f_\omega(\omega_t,z_t,v_t)\\\nonumber
P(|(z_t &+ e_t) + x_t| \leq \omega_t)\geq \alpha,\qquad \forall t\in\mathbb{N}
\end{align}
Given a dataset $\mathcal{D}=\{x_t,u_t,x_{t+1},z_t,v_t,z_{t+1},t\}_{t=0}^N$, we minimize the following loss over data samples $\delta=\{x_t,x_{t+1},z_t,z_{t+1},v_t\}$ in order to learn $f_e(e_t,z_t,v_t)$, which we parameterize with $\xi$:
\begin{equation}
    L_e(\xi,\delta) = \|f_e^\xi(P_\mathbb{Z}(x_t) - z_t,v_t) - P_\mathbb{Z}(x_{t+1}) - z_{t+1}\|_2
    \label{eq:tracking_loss}
\end{equation}
Next, we learn $f_\omega$ by minimizing the quantile loss (\ref{eq:total_loss}).  However, while in the previous section $\omega_t=d(x_t,z_t)$, here we approximate the tube width with $\omega_t=d(x_t,z_t+e_t)$.  We obtain $e_t$ by propagating the learned dynamics $f_e^\xi$ forward in time, given $z_t,v_t$.  Then we can solve a similar tube-based robust MPC problem (\ref{eq:tracking_tube_mpc}):
\begin{subequations}
\label{eq:tracking_tube_mpc}
  \begin{align}
    \min_{v_{\cdot|t}\in\mathbb{V}} & J_T(v_{\cdot|t},z_{\cdot|t}+e_{\cdot|t},\omega_{\cdot|t})\\
    s.t. \,\, \forall k=0,&\cdots,T: \nonumber\\
     z_{k+1|t}&=f_z(z_{k|t},v_{k|t}) \\
     e_{k+1|t}&=f_e^\xi(e_t,z_t,v_t)\\
     \omega_{k+1|t}&=f_\omega^\theta(e_t,z_t,v_t,t)\\
    \omega_{0|t} &= d(x_{t},z_{0|t} + e_{0|t}) \label{eq:tracking_tube_mpc:initial}\\
    z_{T|t} + e_{T|t} &= f_z(z_{T|t},v_{T|t}) + f_e^\xi(e_{T|t},z_{T|t},v_{T|t}) \label{eq:tracking_tube_mpc:terminal1}\\
    \omega_{T|t} & \geq f_\omega^\theta(\omega_{T|t},z_{T|t},v_{T|t},T)\label{eq:tracking_tube_mpc:terminal2}\\
    \Omega^e_{\omega_{k|t}}(z_{k|t})&\subseteq\mathcal{C}
  \end{align}
\end{subequations}

Notice that the cost and constraints are now a function of $z_t+e_t$ and do not depend on $z_t$ only.  This means that we are free to find paths $z_t$ for the tracking controller $\pi$ to track, which may violate constraints.  We maintain the same guarantees of feasibility and constraint satisfaction as in Theorem \ref{thm:mpc}.  Since the proof is similar we omit it for brevity.  See Algorithm \ref{alg:tracking_alg}.
\begin{theorem}
\label{thm:tracking_tube_mpc}
Suppose that the MPC problem (\ref{eq:tracking_tube_mpc}) is feasible at $t=0$.  Then the problem is feasible for all $t>0\in\mathbb{N}$ and at each timestep the constraints are satisfied with probability $\alpha$.
\end{theorem}

\begin{algorithm}[tb]
\small
\textbf{Require:} Ancillary policy $\pi$, Latent dynamics $f_z$,  Safe set $\mathcal{C}$, Quantile probability $\alpha$.  MPC horizon $T$.\\
\textbf{Initialize:} Neural network for error dynamics $f_e^\xi$. Neural network for tube dynamics $f_\omega^\theta$.  Dataset $\mathcal{D}=\{x_{t_i},u_{t_i},x_{t_i+1},z_{t_i},v_{t_i},z_{t_i+1},t_i\}_{i=1}^N$.  Initial states $x_0$, $z_0$, $e_0$, Initial feasible controls $v_{\cdot|0}$. \\
\For{$t=0,\cdots$}{
\If{updateModels}{
    Train $f_e^\xi$ on dataset $\mathcal{D}$ by minimizing error dynamics loss (\ref{eq:tracking_loss}).\\
    Forward propagate learned model $f_x^\xi$ on dataset $\mathcal{D}$ to obtain $\{e_{t_i}\}_{t=1}^N$.  Append to $\mathcal{D}$.\\
    Train $f_\omega^\theta$ on dataset $\mathcal{D}$ by minimizing tube dynamics loss (\ref{eq:total_loss}), but replace $\omega_{t_i}=d(x_{t_i},x_{t_i}+e_{t_i})$.\\
}
\If{$x_{t}$ measured}{
    Initialize tube width $\omega_{0|t}=d(x_t,z_t+e_t)$
}
Solve MPC problem (\ref{eq:tracking_tube_mpc}) with warm-start $v_{\cdot|t}$, obtain $v_t$, $z_{t+1}$\\
Apply control policy to system $u_t=\pi(x_t,z_{t+1},v_t)$\\
Step forward for next iteration: $v_{k|t+1}=v^*_{k+1|t},\ k=0,\cdots,T-1,\ v_{T|t+1}=v^*_{T|t},\ z_{0|t+1}=z^*_{1|t},\ e_{0|t+1}=e^*_{1|t},\  \omega_{0|t+1}=\omega^*_{1|t}$\\
Append data to dataset $\mathcal{D}\leftarrow\mathcal{D}\cup\{x_t,u_t,x_{t+1},z_t,v_t,z_{t+1},t\}$
}
\caption{Learning Tracking Error Dynamics and Tube Dynamics for Tube MPC}
\label{alg:tracking_alg}
\end{algorithm}

\subsection{Learning System Dynamics and Tube Dynamics}
% \vspace{-1cm}
In our third approach to learning tubes, we wish to learn the dynamics directly without a prior nominal model $f_z$.  We restrict $\mathbb{Z}=\mathbb{X}$ and $\mathbb{V}=\mathbb{U}$, and treat $z$ as an approximation of $x$.  Our goal is to learn $f_z$ to approximate $f$, along with $f_\omega$ that will determine a time-varying upper bound on the model error.  Typically the open-loop model error will increase in time in an unbounded manner, which may make it difficult to find a feasible solution to the MPC problem.  One approach is to assume the existence of a stabilizing controller and terminal set, and use a terminal condition that ensures the trajectory ends in this set \cite{kerrigan2001robust,Koller2018}.  A second approach is to find a feedback control law $\pi$ to ensure bounded tube widths.  We describe the latter approach in more detail, but do not restrict ourselves to it.

Using a standard L2 loss function, we first learn an approximation of $f$, call it $f_z^\phi$ with parameters $\phi$:
\begin{equation}
    L_f(\phi,\delta) =  \|f_z^\phi(x_{t},u_{t}) - x_{t+1}\|_2
    \label{eq:loss_dyn}
\end{equation}
Next, we learn a policy $\pi^\psi$ with parameters $\psi$ by inverting the dynamics:
\begin{equation}
    L_\pi(\psi,\delta) = \|\pi^\psi(x_{t},x_{t+1}) - u_{t}\|_2
    \label{eq:loss_policy}
\end{equation}
By learning a policy in this manner we decouple the potentially inaccurate model $f_z^\phi(x_t,u_t)$ from the true dynamics, in a learning inverse dynamics fashion \cite{nguyen2008learning}.  To see this, suppose we have some $z_t$ and $v_t$, and $z_{t+1}=f_z^\phi(z_t,v_t)$.  If $x_t\neq z_t$ and we apply $v_t$ to the real system, $x_{t+1}=f(x_t,v_t)$, then the error $\|x_{t+1}-z_{t+1}\|$ will grow, i.e. $\|x_{t}-z_{t}\| \leq \|x_{t+1}-z_{t+1}\|$.  However, if instead we use the policy $\pi^\psi$, then $f(x_t,\pi^\psi(x_t,z_{t+1}))$ should be closer to $z_{t+1}$, and the error is more likely to shrink.  Other approaches are available for learning $\pi$, including reinforcement learning \cite{sutton2000policy}, imitation learning \cite{ross2011reduction}, etc.
Finally, we learn $f_\omega$ in the same manner as before by minimizing the quantile loss in (\ref{eq:total_loss}).  We generate data for learning the tube dynamics by fitting trajectories of the learned model $f_z^\phi$ to closely approximate the real data $x_t$ (Figure \ref{fig:learn_model_error}).  We randomly initialize $z_{0|t}$ along the trajectory $x_t$ by letting $z_{0|t} = \mathcal{N}(x_{t},\sigma I)$.  We solve the following problem for each $t$:
\begin{subequations}
\label{eq:fit_mpc}
  \begin{align}
    \min_{v_{\cdot|t}\in\mathbb{V}} & \sum_{k=1}^T \|z_{k|t}-x_{t+k}\|\\
    s.t. \quad z_{k+1|t}&=f_z^\phi(z_{k|t},v_{k|t}), \quad \forall k=0,\cdots,T-1
  \end{align}
\end{subequations}
From the fitted dynamics model data, we collect tube training data $\mathcal{D}_z=\bigcup_{t}\Big[\{x_{t+k},x_{t+k+1},z_{k|t},v_{k|t},z_{k+1|t}\}_{k=0}^{T}\Big]$ and proceed to train the tube model.  %For this approach to be effective, the closed-loop dynamics of $x_t$ under the policy $\pi^\psi$ should converge $z_{t+1}$.  Another approach is to apply the policy $\pi^\psi$ to the true system $x_t$, collect real data, and train the tube model.
We can now solve the same tube-based robust MPC problem (\ref{eq:tube_mpc}), with $f_z$ replaced with $f_z^\phi$.  This allows us to maintain the same guarantees of feasibility and safety with probability $\alpha$ as before.  See Algorithm \ref{alg:dynamics_alg}.
\begin{algorithm}[tb]
\small
\textbf{Require:} Safe set $\mathcal{C}$, Quantile probability $\alpha$.  MPC horizon $T$.\\
\textbf{Initialize:} Neural network for policy $\pi^\psi$, dynamics $f_z^\phi$, and tube dynamics $f_\omega^\theta$.  Dataset $\mathcal{D}=\{x_{t_i},u_{t_i},x_{t_i+1}\}_{i=1}^N$.  Initial state $x_0$.\\
Solve MPC problem (\ref{eq:tube_mpc}) for initial feasible control sequence $v_{\cdot|0}$.\\
\For{$t=0,\cdots$}{
\If{updateModel}
{
    Train $f_z^\phi$ on dataset $\mathcal{D}$ by minimizing dynamics loss (\ref{eq:loss_dyn}).\\
    Train $\pi^\psi$ on dataset $\mathcal{D}$ by minimizing policy loss (\ref{eq:loss_policy}).\\
    Create $\mathcal{D}_z=\bigcup_{t}\Big[\{x_{t+k},x_{t+k+1},z_{k|t},v_{k|t},z_{k+1|t}\}_{k=0}^{T}\Big]$ by solving (\ref{eq:fit_mpc}). \\
    Train $f_\omega^\theta$ on dataset $\mathcal{D}_z$ by minimizing tube dynamics loss (\ref{eq:total_loss}).
}
\If{$x_{t}$ measured}{
    Initialize tube width $\omega_{0|t}=d(x_t,z_t)$ 
}
Solve MPC problem (\ref{eq:tube_mpc}) with warm-start $v_{\cdot|t}$, obtain $v_t$, $z_{t+1}$\\
Apply control policy to system $u_t=\pi^\psi(x_t,z_{t+1})$\\
Step forward for next iteration: $v_{k|t+1}=v^*_{k+1|t},\ k=0,\cdots,T-1,\ v_{T|t+1}=v^*_{T|t},\ z_{0|t+1}=z^*_{1|t},\ \omega_{0|t+1}=\omega^*_{1|t}$\\
Append data to dataset $\mathcal{D}\leftarrow\mathcal{D}\cup\{x_t,u_t,x_{t+1}\}$
}
\caption{Learning Dynamics and Model Error Bounds for Tube MPC}
\label{alg:dynamics_alg}
\end{algorithm}
%===============================================================================
\section{Experimental Details}
\label{sec:4}
\subsection{Evaluation on a 6-D problem}
In this section we validate each of our three approaches to learned tubes for tube MPC on a 6-state simulated triple-integrator system.  We introduce two sets of dynamics for $f$ and $f_z$ to demonstrate our method.  Consider the following 2D triple-integrator system with 6 states, where $x=[p_x,p_y,v_x,v_y,a_x,a_y]^\intercal$, along with the 4 state 2D double-integrator dynamics for the reference system: $z=[p^z_x,p^z_y,v^z_x,v^z_y]$.  Let these systems have the following dynamics (we show the $x$-axis only for brevity sake):
\begin{align}
    \frac{d}{dt}\begin{bmatrix}p_x\\v_x\\a_x\end{bmatrix}
                &=\begin{bmatrix}0 & 1 & 0 \\
                     0 & 0 & 1 \\
                     0 & 0 & -k_f\end{bmatrix}\begin{bmatrix}p_x\\v_x\\a_x\end{bmatrix}  + \begin{bmatrix}0\\0\\1\end{bmatrix} u_x + \begin{bmatrix}0 & 0\\1 & 0\\0 & 1\end{bmatrix} w\\
    \frac{d}{dt}\begin{bmatrix}p^z_x\\v^z_x\end{bmatrix}&=\begin{bmatrix}0 & 1 \\
                     0 & -k_f^z\end{bmatrix}\begin{bmatrix}p^z_x\\v^z_x\end{bmatrix}  + \begin{bmatrix}0\\1\end{bmatrix} v_x
\label{eq:triple_integrator}
\end{align}
where $w\sim\mathcal{N}(0,\epsilon I_{2\times2})$, and with similar dynamics for the $y$-axis.  We construct the following cascaded PD control law:
\begin{align}
    \pi_x(p_x,p^z_x) &= k_d(k_p(p^z_x-p_x)-v_x + v^z_x) + k_a(-a_x)
\end{align}
We choose $k_f=0.1,k_f^z=1.0,k_p=1,k_d=10,k_a=5$, and $\epsilon=0.05$.  We also bound $\|v_x\|,\|v_y\|\leq 1$.  We simulate in discrete time with $dt=0.1$.

We collect $\sim$100 episodes with randomly generated controls, with episode lengths of $\sim$100 steps.  Following each algorithm, we then set up an MPC task to navigate through a forest of obstacles (see Figure \ref{fig:mpc_comparison}).  We found an MPC planning horizon of 20-30 steps to be effective.  We ran each MPC algorithm for 100 steps, or until the system reaches the goal.  We also plot 100 rollouts of the "true" system $x_t$ to evaluate the learned bounds.  For each learned network, we use 3 layers with 256 units each.  When calculating constraints for the tube, we treat the tube width $\omega_t$ as axes for an ellipse rather than a box. This alleviates the need for solving a mixed integer quadratic program, at the cost of a slightly larger tube.  We use a quadratic running cost that penalizes deviation from the goal and excessively large velocities.

With Algorithm \ref{alg:tube_alg}, we note that the tube widths are quite large.  This is because this algorithm uses the reference trajectory itself as the center of the tube.  While the tube encloses the trajectories, it does not create a tight bound.  In Algorithm \ref{alg:tracking_alg}, we address this issue directly.  We learn dynamics of the mean tracking error and use this as our tube center.  The resulting tube dynamics bound the state distribution more closely.  Note that when solving the MPC problem, the optimized reference trajectory $z_t$ is free to violate the constraints, as long as the system trajectories $x_t$ do not.  This approach allows for much more aggressive behaviors.  For Algorithm 3, without a good tracking controller, the tube width increases over time.  However, because we replan at each timestep with a finite horizon, the planner is still able to fit through narrow passages.  In the example shown we replan from the current state $x_t$, with the assumption that it is measured.  This allows us to create aggressive trajectories with narrow tube widths.
\begin{figure}[tb]%
    \centering
    \includegraphics[width=.9\linewidth,trim={0 25 0 15},clip]{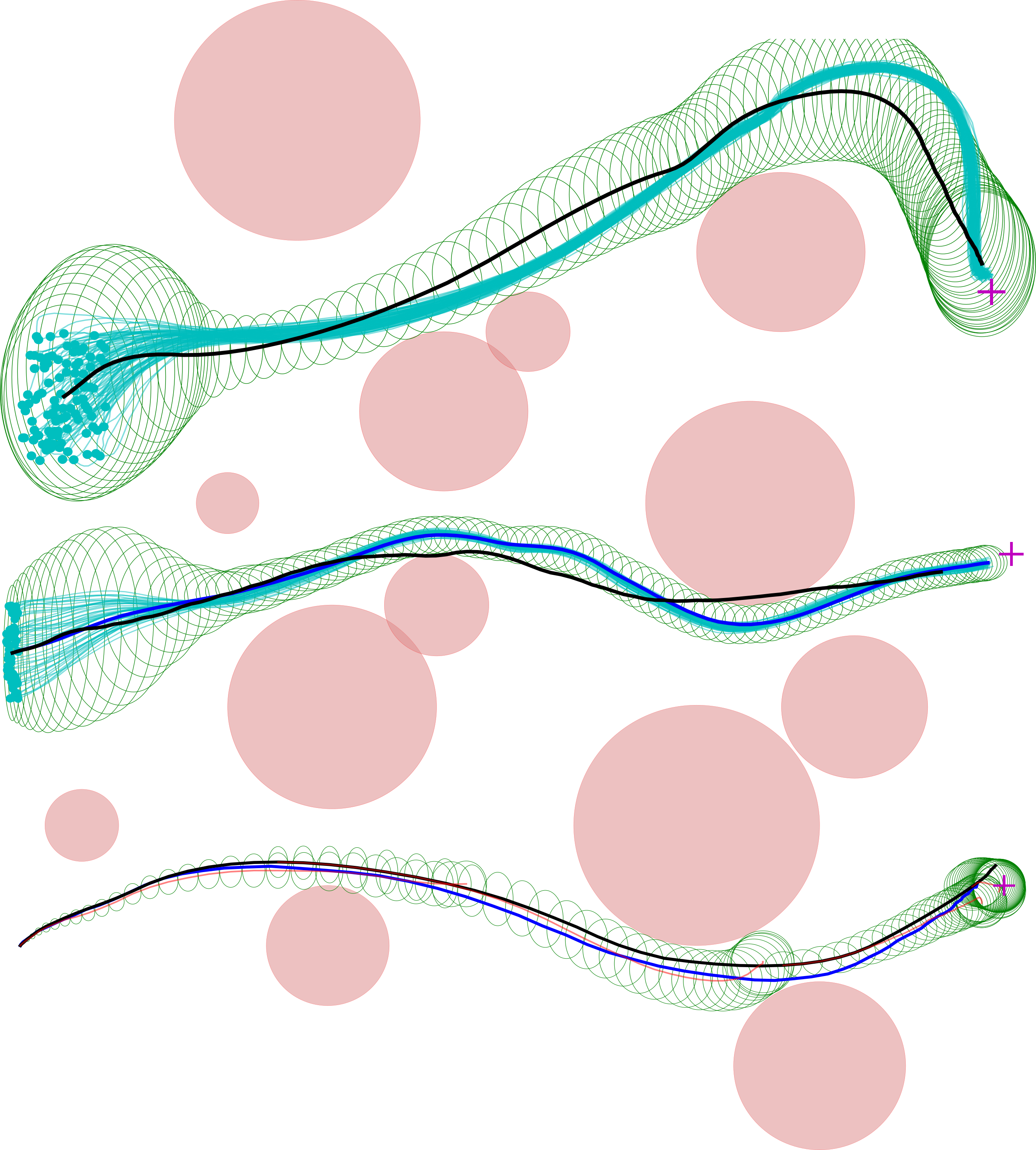}
    \caption{{Comparison of 3 tube MPC approaches with learned tubes.  Red circles denote obstacles, magenta cross denotes goal.  Cyan lines indicate sampled trajectories from the system $x_t$ with randomized initial conditions.    Top: Algorithm 1, learning a tube around the reference $z$ (black) used for tracking.  Green circles indicate the tube width obtained at each timestep.  Mid: Algorithm 2, learning tracking error dynamics (blue line) for the center of the tube.  Bot: Algorithm 3, tube MPC problem using learned policy, dynamics, and tube dynamics.  Red lines indicate planned NN dynamics trajectories at each MPC timestep, along with the forward propagated tube dynamics (green), shown every 20 timesteps.  Blue line indicates actual path taken ($x_t$).}}%
    \label{fig:mpc_comparison}
\end{figure}
\subsection{Comparison with analytic bounds}
We compare our learned tubes with an analytic solution for robust bounds on the system (\ref{eq:triple_integrator}).  We derive these analytic bounds by assuming worst-case noise perturbations of the closed-loop system.  We find the bound $W$ such that $P(|w_t| \leq W)\geq \alpha$ (with $\alpha=0.95$).  The worst-case error at each timestep is $w_t=\pm W$.  We compare these bounds with those learned with our quantile method (Figure \ref{fig:bound_comparison}).  Our method tends to underestimate the true bounds slightly, which is due to the training data rarely containing worst-case adversarial noise sequences.  %While analytic bounds may provide more accurate estimates of the distribution bounds, they are not always available.  On real-world systems with non-linearity, modeling error, time-varying dynamics, and disturbances, a learning-based approach may be more desirable.

\begin{figure}[tb]
\centering
\subfloat{
    \includegraphics[width=0.6\linewidth,trim={15 30 15 0},clip]{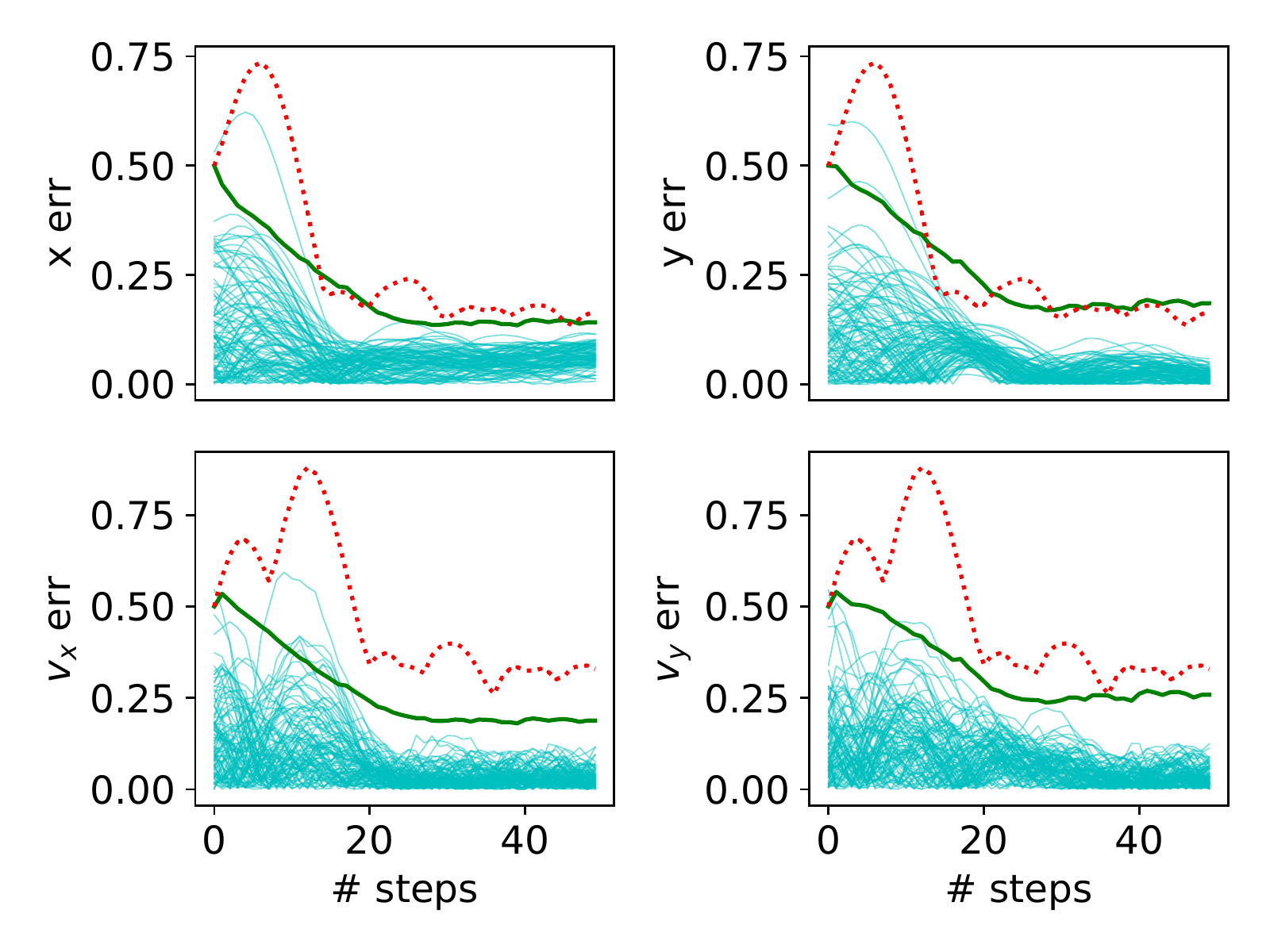}
}
% \subfloat{
    % \includegraphics[width=0.3\linewidth]{figures/pointmass_example.pdf}
% }
\caption{Learned $95\%$ quantile error bounds (green) vs. $95\%$ analytic bounds (dotted red) for the linear triple-integrator system, with 100 sampled trajectories, tracking a random reference trajectory.}
\label{fig:bound_comparison}
\end{figure}

\subsection{Ablative Study}
We perform an ablative study of our tube learning method.  Using Algorithm \ref{alg:tube_alg}, we learn error dynamics and tube dynamics.  We collect randomized data (400 episodes of 40 timesteps) and train $f_\omega$ under varying values of $\alpha$.  We then evaluate the accuracy of $f_\omega$ by sampling 100 new episodes of 10 timesteps, and plot the frequency that $f_\omega$ overestimates the true error, along with the magnitude of overestimation (Figure \ref{fig:alpha}, left).  We compare networks learned with the epistemic loss and without it, and find that our method produces well-calibrated uncertainties when using the epistemic loss, along with the quantile and monotonic losses (\ref{eq:total_loss}).  We evaluated ablation of the monotonic loss but found no noticeable differences.

We also evaluate estimation of epistemic uncertainty with varying amounts of data (from 10 to 400 episodes), with a fixed value of $\alpha=0.95$.  We find that estimating epistemic uncertainty is particularly helpful in the low-data regimes (Figure \ref{fig:alpha}, right).  As expected, the network maintains good quantile estimates by increasing the value of $f_\omega$, which results in larger tubes.  This creates more conservative behavior when the model encounters new situations.
\begin{figure}[tb]
\centering
\includegraphics[width=\linewidth]{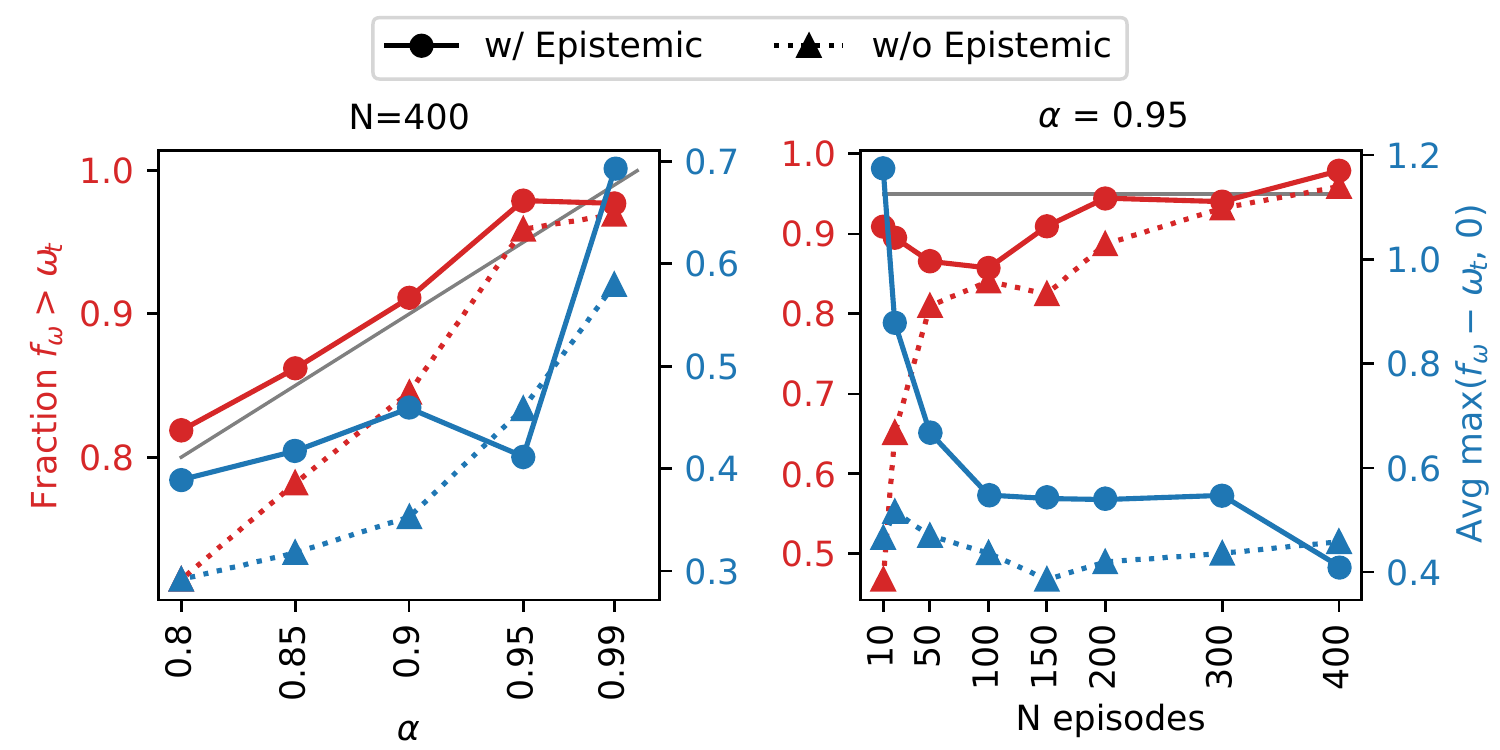}
\caption{Evaluation of learned tube dynamics $f_\omega$ on triple integrator system with varying $\alpha$ (left) and varying number of datapoints (right).  Red indicates fraction of validation samples that exceed the bound, while blue indicates average distance in excess of the bound.  Models learned with the epistemic loss along with the quantile loss (circles, solid lines) perform better vs. models without epistemic uncertainty (triangles, dotted lines).  Gray lines mark the best possible values.}
\label{fig:alpha}
\end{figure}
\subsection{Evaluation on Quadrotor Dynamics}
To validate our approach scales well to high-dimensional non-linear systems, we apply Algorithm \ref{alg:tracking_alg} to a 12 state, 4 input quadrotor model, with dynamics:
\begin{align*}
    \dot{\mathbf{x}} &= \mathbf{v} &  m\dot{\mathbf{v}}&=mg\mathbf{e}_3-T R\mathbf{e}_3\\
    \dot{R} &= R\hat{\Omega} &  J\dot{\Omega} &= M + w - \Omega \times J\Omega
\end{align*}
where $\hat{\cdot}:\mathbb{R}^3\rightarrow SO(3)$ is the hat operator.  The states are the position $\mathbf{x}\in\mathbb{R}^3$, the translational velocity $\mathbf{v}\in\mathbb{R}^3$, the rotation matrix from body to inertial frame $R\in SO(3)$, and the angular velocity in the body frame $\Omega\in\mathbb{R}^3$.  $m\in\mathbb{R}$ is the mass of the quadrotor, $g\in\mathbb{R}$ denotes gravitational force, and $J\in\mathbb{R}^{3\times3}$ is the inertia matrix in body frame.  The inputs to the model are the total thrust $T\in\mathbb{R}$ and the total moment in the body frame $M\in\mathbb{R}^3$.  Noise enters through the control channels, with $w\sim\mathcal{N}(0,\epsilon I_{3\times3})$.  Our state is $x_t=\{\mathbf{x},\mathbf{v},R,\Omega\}\in\mathbb{R}^{18}$ and control input is $u_t=\{T,M\}\in\mathbb{R}^4$.  We use a nonlinear geometric tracking controller that consists of a PD controller on position and velocity, which then cascades to an attitude controller \cite{lee2010geometric}.  For the nominal model $f_z$ we use a double integrator system on each position axis.  The nominal state is $z_t=\{\mathbf{x},\mathbf{v}\}\in\mathbb{R}^6$ with acceleration control inputs $v_t=\{a_x,a_y,a_z\}\in\mathbb{R}^3$.  See Figures \ref{fig:quad} and \ref{fig:quad_detail}.
% For this quadrotor and tracking controller system the tracking error can be quite large; therefore we prefer Algorithm \ref{alg:tracking_alg} to \ref{alg:tube_alg}, although Algorithm \ref{alg:tube_alg} has fewer constraints and may be faster.  Algorithm \ref{alg:dynamics_alg} is more suited to learning dynamics and disturbances in the absence of a good tracking controller and good nominal model.  For the quadrotor system this could involve learning the full 18-dimensional dynamics which may be difficult.  (A potential workaround would be to learn a reduced order model and capture model discrepancies with a tube).  Nevertheless, we choose Algorithm \ref{alg:tracking_alg} and show it is able to scale up to this system (Figures \ref{fig:quad}, \ref{fig:quad_detail}).
\begin{figure}[tb]
\centering
\subfloat{
    \includegraphics[width=0.9\linewidth]{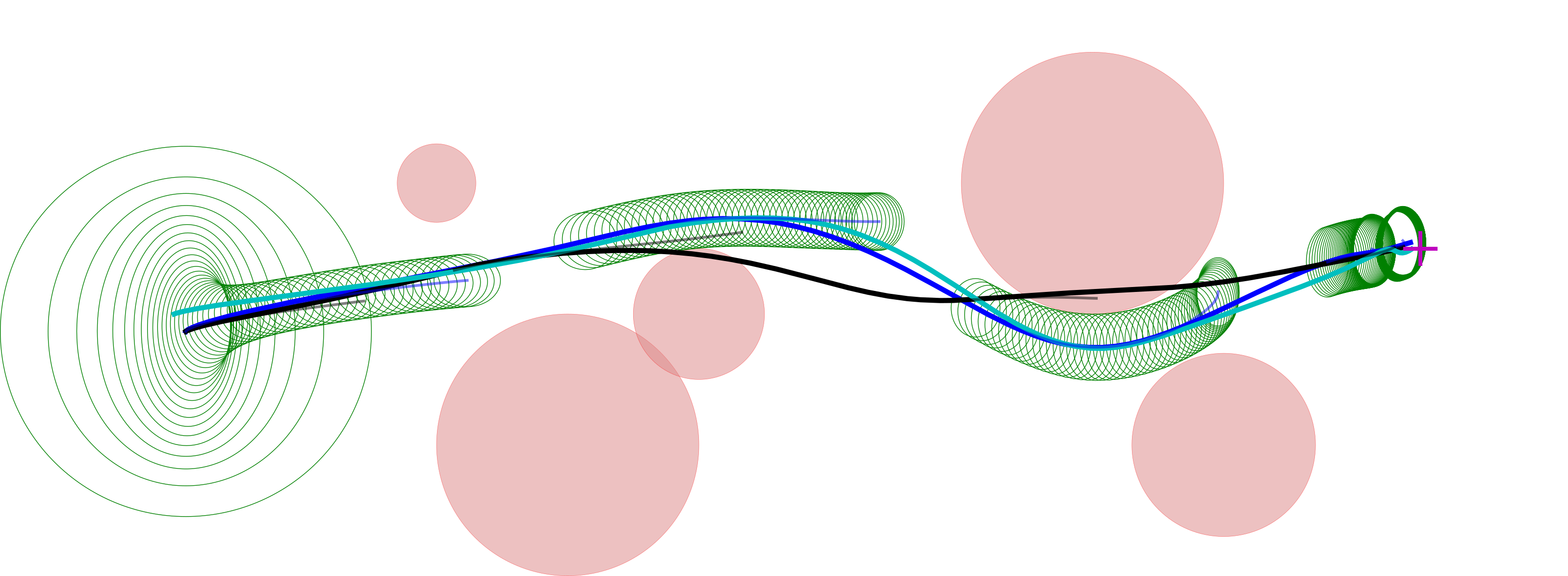}
}%\\[-2ex]
% \subfloat{
%     \includegraphics[trim=2cm 0cm 1cm 0cm, clip,width=\linewidth]{figures/quadrotor_mpc.pdf}
% }
\caption{Algorithm \ref{alg:tracking_alg} working on quadrotor dynamics, showing 5 individual MPC solutions at different times along the path taken.  Thinner lines (black and blue) indicate planned future trajectories $z_{\cdot|t}$ and $e_{\cdot|t}$, respectively.}%  Bottom:  Overall path taken at each timestep, green tube indicates $\omega_{0|t}$ for each $t$.  Cyan lines indicate 100 sampled trajectories tracking $z_t$ with different initial conditions.  Notice that $z_t$ is free to violate the constraints, while the actual cyan trajectories $x_t$ do not.}
\label{fig:quad}
\end{figure}
\begin{figure}[tb]
\centering
\includegraphics[trim=1cm 0.5cm 1cm 1cm, clip,width=0.8\linewidth]{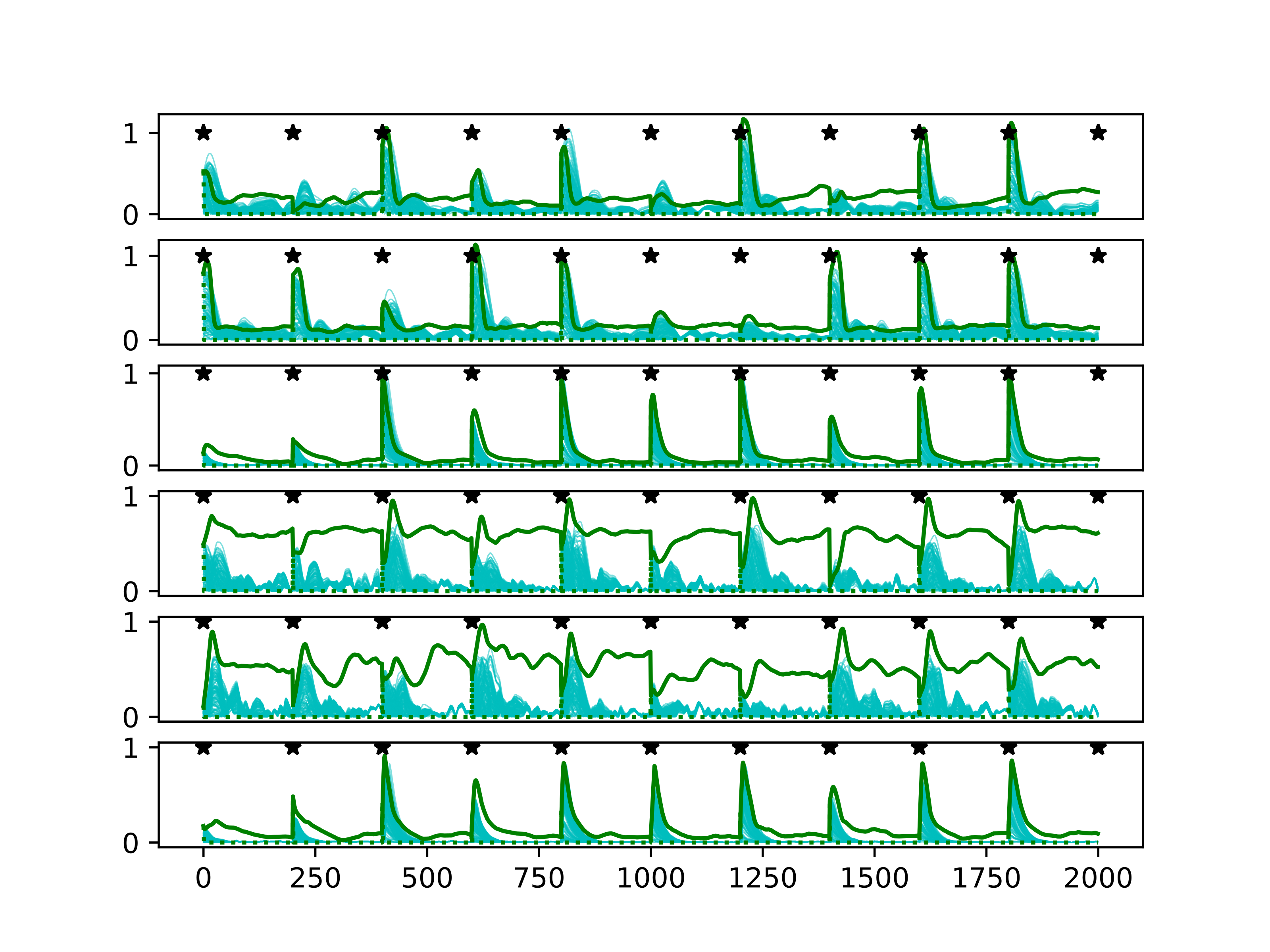}
\caption{Tube widths $f_\omega$ for quadrotor dynamics, 10 episodes of 200 timesteps each, tracking random reference trajectories. From top to bottom, we plot $(p_x,p_y,p_z,v_x,v_y,v_z)$.  Green lines indicate the quantile bound $\omega_t$, with $\alpha=0.9$, and cyan lines show 100 sampled error trajectories, $|x_t-e_t|$.  Black stars indicate the start of a new episode.}
\label{fig:quad_detail}
\end{figure}
% Calculating the gradients of the neural network models for the linearization step takes on the order of 200ms, while solving the SQP problem takes around 100ms.  We expect these times to be easily reduced with a more optimal implementation geared towards real-time capabilities.  Other real-time MPC methods may be utilized as well, including GPU-accelerated methods such as MPPI \cite{williams2018information}.  We intend to make our code publicly available pending peer review.
% \footnote{\url{http://bit.ly/f8awsff}}.
\section{Conclusion}
\label{sec:5}
We have introduced a deep quantile regression framework for learning bounds on controlled distributions of trajectories.  For the first time we combine deep quantile regression in three robust MPC schemes with recursive feasibility and constraint satisfaction guarantees.  We show that these schemes are useful for high dimensional learning-based control on quadrotor dynamics.  We hope this work paves the way for more detailed investigation into a variety of topics, including deep quantile regression, learning invariant sets for control, handling epistemic uncertainty, and learning-based control for non-holonomic or non-feedback linearizable systems.  Our immediate future work will involve hardware implementation and evaluation of these algorithms on a variety of systems.
%===============================================================================
\section*{Acknowledgement}
We thank Brett Lopez and Rohan Thakker for insightful discussions and suggestions, as well as the reviewers for helpful comments.  This research was partially carried out at the Jet Propulsion Laboratory (JPL), California Institute of Technology, and was sponsored by the JPL Year Round Internship Program and the National Aeronautics and Space Administration (NASA).  Copyright \textcopyright 2020. All rights reserved.
% \newpage
\bibliographystyle{plainnat}
\bibliography{main}

\end{document}